\documentclass[letterpaper]{article}
 \usepackage{aaai}
 \usepackage{times}
 \usepackage{helvet}
 \usepackage{courier}
 
 \usepackage{epsfig}
\usepackage{graphicx}
\usepackage{wrapfig}
\usepackage{xspace}
\usepackage{amsmath}
\usepackage{flushend}

\newtheorem{theorem}{Theorem}[section]
\newtheorem{lemma}[theorem]{Lemma}
\newtheorem{definition}[theorem]{Definition}

\newtheorem{corollary}[theorem]{Corollary}
\newenvironment{RETHM}[2]{\it \trivlist \item[\hskip \labelsep{\bf #1 \ref{#2}}]}{\endtrivlist}
\newcommand{\rethm}[1]{\begin{RETHM}{Theorem}{#1}}
\newcommand{\repro}[1]{\begin{RETHM}{Proposition}{#1}}
\newcommand{\relem}[1]{\begin{RETHM}{Lemma}{#1}}
\newcommand{\recor}[1]{\begin{RETHM}{Corollary}{#1}}

\newcommand{\erethm}{\end{RETHM}}
\newcommand{\erepro}{\end{RETHM}}
\newcommand{\erelem}{\end{RETHM}}
\newcommand{\erecor}{\end{RETHM}}

\newcommand{\fp}{$\mbox{FP}^{{\rm NP}[\log{n}]}$}

\newcommand{\fpsigma}{$\mbox{FP}^{\Sigma_2^P[\log{n}]}$}
\newcommand{\fpsigmad}{$\mbox{FP}^{\D_2[\log{n}]}$}

\newcommand{\fpa}{$\mbox{FP}^{{\rm A}[\log{n}]}$}

\newcommand{\fpapar}{$\mbox{FP}^{\rm A}_{||}$}
\newcommand{\fpsigmapar}{$\mbox{FP}^{\Sigma_2^P}_{||}$}
\newcommand{\fpsigmapard}{$\mbox{FP}^{\D_2}_{||}$}
\newcommand{\fpsigmapars}{\fpsigmapar\ }
\newcommand{\fppar}{$\mbox{FP}^{{\rm NP}}_{||}$}

\newcommand{\pair}[1]{\langle #1  \rangle}

\def\squarebox#1{\hbox to #1{\hfill\vbox to #1{\vfill}}}
\newcommand{\qed}{\hspace*{\fill}
            \vbox{\hrule\hbox{\vrule\squarebox{.667em}\vrule}\hrule}\smallskip}
\newenvironment{proof}{\begin{trivlist}
\item[\hspace{\labelsep}{\bf\noindent Proof: }]
}{\qed\end{trivlist}}

\renewenvironment{proof}{\begin{trivlist}
\item[\hspace{\labelsep}{\bf\noindent Proof: }]
}{\qed\end{trivlist}}

\newenvironment{proofout}{\begin{trivlist}
\item[\hspace{\labelsep}{\bf\noindent Proof outline: }]
}{\qed\end{trivlist}}

\newtheorem{xmpl}[theorem]{Example}
\newenvironment{example}{\begin{xmpl}\rm}{\end{xmpl}}

\newtheorem{rmark}[theorem]{Remark}

\newcommand{\U}{{\cal U}}
\newcommand{\D}{{D}}

\newcommand{\cS}{{\cal S}}

\newcommand{\Lan}{L}

\newcommand{\zug}[1]{\langle #1  \rangle}

\newcommand{\stam}[1]{}

\newcommand{\SH}{\mbox{{\it SH}}}
\newcommand{\BH}{\mbox{{\it BH}}}
\newcommand{\BT}{\mbox{{\it BT}}}
\newcommand{\BS}{\mbox{{\it BS}}}
\newcommand{\ST}{\mbox{{\it ST}}}

\newcommand{\bd}{\begin{definition}}
\newcommand{\ed}{\end{definition}}
\newcommand{\be}{\begin{enumerate}}
\newcommand{\bi}{\begin{itemize}}
\newcommand{\ee}{\end{enumerate}}
\newcommand{\ei}{\end{itemize}}
\newcommand{\T}{\mbox{\sc \bf true}\xspace}

\newcommand{\Lcbin}{L_{\mbox{\small cause}}^B}
\newcommand{\Lc}{L_{\mbox{\small cause}}}
\newcommand{\Lsingle}{L_{\mbox{\small AC2}}}
\newcommand{\Lmin}{L_{\mbox{\small AC3}}}
\newcommand{\Lsinglebin}{L_{\mbox{\small AC2}}^B}
\newcommand{\Lminbin}{L_{\mbox{\small AC3}}^B}
\newcommand{\Lcbino}{L_{\mbox{\small cause}}^{B,1}}
\newcommand{\Lco}{L_{\mbox{\small cause}}^1}
\newcommand{\Lsingleo}{L_{\mbox{\small AC2}}^1}
\newcommand{\Lmino}{L_{\mbox{\small AC3}}^1}
\newcommand{\Lsinglebino}{L_{\mbox{\small AC2}}^{B,1}}
\newcommand{\Lminbino}{L_{\mbox{\small AC3}}^{B,1}}

\newcommand{\cF}{{\cal F}}
\newcommand{\V}{{\cal V}}
\newcommand{\R}{{\cal R}}

\newcommand{\DP}{D^P}
\renewcommand{\citeyear}{\shortcite}
\newcommand{\commentout}[1]{}
\newcommand{\fullv}[1]{#1}
\newcommand{\shortv}[1]{\commentout{#1}}

\renewcommand{\phi}{\varphi}

\newcommand{\dr}{\mbox{{\em dr}}}
\newcommand{\db}{\mbox{{\em db}}}
\newcommand{\K}{{\cal K}}
\newcommand{\NP}{\mbox{{\it NP}}}

\title{The Computational Complexity of Structure-Based Causality}
\author{Gadi Aleksandrowicz \\
IBM Research Lab, \\ Haifa, Israel \\
gadia@il.ibm.com \\
\And
Hana Chockler \\
Department of Informatics,\\
King's College, \\ London, UK\\
hana.chockler@kcl.ac.uk \\
\And
Joseph Y. Halpern\\
Computer Science Department, \\
Cornell University,\\
Ithaca, NY, U.S.A. \\
halpern@cs.cornell.edu \\
\And Alexander Ivrii \\
IBM Research Lab, \\ Haifa, Israel\\
alexi@il.ibm.com}

\begin{document}
\maketitle

\begin{abstract}
Halpern and Pearl 
introduced a definition of
actual causality; Eiter and Lukasiewicz 
showed that
computing whether $X=x$ is a cause of $Y=y$ is $\NP$-complete in binary
models (where all variables can take on only two values) and
$\Sigma_2^P$-complete in general models.  In the final version of their
paper, 
Halpern and Pearl slightly modified the definition of
actual cause, in order to deal with problems pointed by Hopkins and
Pearl.
As we show, this modification has a
nontrivial impact on the complexity of computing actual cause.
To characterize the complexity, a new family $D_k^P$, $k= 1, 2, 3,
\ldots$, of complexity classes is introduced, which generalizes the class
$\DP$ introduced by 
Papadimitriou and Yannakakis
($\DP$ is just $D_1^P$).
We show that the complexity of computing causality 
under the updated definition is $D_2^P$-complete.  

Chockler and Halpern
extended the 
definition of causality by introducing notions of \emph{responsibility}
and \emph{blame}. 
The complexity of 
determining the 
degree of responsibility and blame using the original definition of
causality
was completely characterized. 
Again, we show that changing the definition of causality
affects the complexity, and completely characterize it
using the updated definition.
\end{abstract}

\section{Introduction}
There have been many attempts to define {\em causality\/} going back to
Hume \citeyear{Hume39}, and continuing to the present (see, for example,
\cite{Collins03,pearl:2k} for some recent work). 
The standard definitions of causality are based on
counterfactual reasoning. In this paper, we focus on one such definition,
due to Halpern and Pearl, that has proved quite influential recently.

The definition was originally introduced in 2001 \cite{HPearl01a}, but
then modified in the final journal version \cite{HP01b} to deal with
problems pointed out by Hopkins and Pearl \citeyear{HopkinsP02}.  
(For ease of reference, we call these definitions ``the original HP
definition'' and ``the updated HP definition'' in the sequel.)
In general, what can be a cause in both the original HP definition and
the updated definition is a conjunction of the form $X_1 \gets x_1 \land
\ldots \land X_k \gets x_k$, abbreviated $\vec{X} \gets \vec{x}$; what
is caused can be an arbitrary Boolean 
combination $\phi$ of formulas of the form $Y = y$.  This should be
thought of as saying that setting $X_1$ to $x_1$ and~$\ldots$ and setting
$X_k$ to $x_k$ results in $\phi$ being true.  As shown by Eiter and
Lukasiewicz \citeyear{EL01} and Hopkins \citeyear{Hopkins01}, under the
original HP definition, we can always take causes to be single
conjuncts.  However, as shown by Halpern \citeyear{Hal39}, this is not
the case for the updated HP definition.

Using the fact that causes can be taken to be single conjuncts,
Eiter and Lukasiewicz\citeyear{EL01} showed that deciding causality
(that is, deciding whether $X=x$ is a cause of $\phi$) 
is $\NP$-complete in binary
models (where all variables can take on only two values) and
$\Sigma_2^P$-complete in general models.  As we show here, this is no
longer the case for the updated HP definition.  Indeed, we completely
characterize the complexity of causality for the updated HP definition.
To do so, we introduce a new family of complexity classes that may be of
independent 
interest.  Papadimitriou and Yannakakis \citeyear{PY84} introduced the
complexity class $\DP$, which consists of all languages $\Lan_3$ such that
there exists a language $\Lan_1$ in $\NP$ and a language $\Lan_2$ in
co-$\NP$ such 
that $\Lan_3 = L_1 \cap 
\Lan_2$.
 We generalize this by defining $D_k^P$ to consist of all
languages $\Lan_3$ such that there exists a language $\Lan_1 \in \Sigma_k^P$
and a language $\Lan_2 \in \Pi_k^P$ such  that $\Lan_3 = \Lan_1 \cap
\Lan_2$. 

Since $\Sigma_1^P$ is $\NP$ and $\Pi_1^P$ is co-$\NP$,
$D_1^P$ is Papadimitriou and Yannakakis's $\DP$.  We then show that 
deciding causality under the updated HP definition is $D_2^P$ complete.
Papadimitriou and Yannakakis \citeyear{PY84} showed that a number of
problems of interest were $\DP$ complete, both for binary and general
causal models.  To the best of our
knowledge, this is the first time that a natural problem has been shown
to be complete for $D_2^P$.

Although, in general, causes may not be single conjuncts, as observed by
Halpern \citeyear{Hal39}, in many cases (in particular, in all the
standard examples studied in the literature), they are.  In an effort to
understand the extent to which the difficulty in deciding causality
stems from the fact that causes may require several conjuncts, we
consider what we call the \emph{singleton cause} problem; that is, the
problem of deciding if $X=x$ is a cause of $\phi$ (i.e., where there is
only a single conjunct in the cause).  We show that the singleton cause
problem is simpler than the general causality problem (unless the
polynomial hierarchy collapses): it is $\Sigma_2^P$ complete for both
binary and general causal models.  
\fullv{ Thus, if we restrict to singleton
causes (which we can do without loss of generality under the original HP
definition), the complexity of deciding causality in general models is
the same under the 
original and the updated HP definition, but in binary models, it is still
simpler under the original HP definition.
}
\shortv{
Thus, if we restrict to singleton
causes, the complexity of deciding causality in general models is
the same under the 
original and the updated HP definition, but in binary models, it is still
simpler under the original HP definition.
}

Causality is a ``0--1'' concept; $\vec{X} = \vec{x}$ is either a cause of
$\phi$ or it is not.  Now consider two voting scenarios: in the first, Mr.~G
beats Mr.~B by a vote of 11--0.  In the second, Mr.~G beats Mr.~B by a
vote of 6--5.  According to both the original and the updated HP
definition, all the people who voted for Mr. G are causes of him
winning.  While this does not seem so unreasonable, it does not capture
the intuition that each voter for Mr. G is more critical to the
victory in the case of the 6--5 vote than in the case of the 11--0 vote.
The notion of \emph{degree of responsibility}, introduced by Chockler and
Halpern \citeyear{CH04}, does so.  The idea is that the degree of
responsibility of $X=x$ for $\phi$ is $1/(k+1)$, where $k$ is the least
number of changes that have to be made in order to make $X=x$ critical.
In the case of the 6--5 vote, no changes have to be made to make each
voter for Mr.~G critical for Mr.~G's victory; if he had not voted for
Mr.~G, Mr.~G would not have won.  Thus, each voter has degree of
responsibility 1 (i.e., $k=0$).  On the other hand, in the case of the
11--0 vote, for a particular voter to be critical, five other voters
have to switch their votes; thus, $k=5$, and each voter's degree of
responsibility is $1/6$.  This notion of degree of responsibility has
been shown to capture (at a qualitative level) the way people allocate
responsibility \cite{GL10,LGZ13}.  

Chockler and Halpern further extended the notion of degree of
responsibility to \emph{degree of blame}.  Formally, the degree of blame
is the expected degree of responsibility.  This is perhaps best
understood by considering a firing squad with ten excellent marksmen.  
Only one of them has live bullets in his rifle; the rest have blanks.  
The marksmen do not know which of them has the live bullets.  The
marksmen shoot at the prisoner and he dies.  The only marksman that is
the cause of the prisoner's death is the one with the live bullets.
That marksman has degree of responsibility 1 for the death; all the rest
have degree of responsibility 0.  However, each of the marksmen has
degree of blame $1/10$.%
%
The complexity of
determining the degree of responsibility and blame 
using the original
definition of causality 
was completely characterized \cite{CH04,CHK}.
Again, we show that changing the definition of
causality affects the complexity, and completely characterize the
complexity of determining the degree of responsibility and blame with
the updated definition.

The rest of this paper is organized as follows.  In
Section~\ref{sec:definitions}, we review the relevant definitions of
causality.  In Section~\ref{sec:complexity}, we 
briefly review the
relevant definitions from complexity theory and define the complexity
classes $D_k^P$.  In Section~\ref{sec-cause} we prove our results on
complexity of causality.%
\shortv{\footnote{Missing proof details can be found at http://www.cs.cornell.edu/home/halpern/papers/newcause.pdf.}
}
\fullv{Some proofs are deferred to the appendix.}

\section{Causal Models and Causality: A Review}\label{sec:definitions}

In this section, we review the details of 
Halpern and Pearl's definition of causal models and causality,
describing both the original definition and the updated definition.
This material is largely taken from \cite{HP01b}, to which we refer the
reader for further details.  

\subsection{Causal models}

A {\em signature\/} is a tuple $\cS = \zug{\U,\V,\R}$, 
where $\U$ is a finite set
of {\em exogenous\/} variables, $\V$ is a 
finite
set of {\em endogenous\/}
variables,  
and $\R$ associates with every variable  
$Y \in \U \cup \V$ a 
finite
nonempty set $\R(Y)$ of possible values for $Y$.
Intuitively, the  exogenous variables are ones whose values are
determined by factors outside the model, while the endogenous variables
are ones whose values are ultimately determined by the exogenous
variables.
A {\em causal model\/} over signature $\cS$ is a tuple
$M = \zug{\cS,\cF}$, where $\cF$ associates with every endogenous variable
$X \in \V$ a function $F_X$ such that 
$F_X: (\times_{U \in \U} \R(U) \times (\times_{Y \in \V \setminus \{ X \}}
\R(Y))) \rightarrow \R(X)$. That is, $F_X$ describes how the value of the
endogenous variable $X$ is determined by
the values of all other variables in $\U \cup \V$. 
If $\R(Y)$ contains only two values for each $Y \in \U \cup \V$, then 
we say that $M$ is a 
{\em binary causal model}.

We can describe (some salient features of) a causal model $M$ using a
{\em causal network}. A causal network
is a graph
with nodes corresponding to the random variables in $\V$ and an edge
from a node labeled $X$ to one labeled $Y$ if $F_Y$ depends on the value
of $X$.
Intuitively, variables can have a causal effect only on their
descendants in the causal network; if $Y$ is not a descendant of $X$,
then a change in the value of $X$ has no affect on the value of $Y$.
For ease of exposition, 
we restrict attention to what are called {\em
recursive\/} models. These are ones whose associated causal network
is a directed acyclic graph (that
is, a graph that has no  cycle of edges).
Actually, it suffices for our purposes that, 
for each setting $\vec{u}$
for the variables in $\U$, there is no cycle among the edges of the causal
network.   
We call a setting $\vec{u}$ for the variables in $\U$ a {\em context}.
It should be clear that if $M$ is a recursive causal model,
then there is always a
unique solution to the equations in $M$, given a context.

The equations determined by $\{F_X: X \in \V\}$ can be thought of as
representing  processes (or mechanisms) by which values are assigned to
variables.  For example, if $F_X(Y,Z,U) = Y+U$ (which we usually write 
as $X=Y+U$), then if $Y = 3$ and $U = 2$, then $X=5$,
regardless of how $Z$ is set.
This equation also gives counterfactual information. It says that,
in the context $U = 4$, if $Y$ were $4$, then $X$ would
be $8$, regardless of what value $X$ and $Z$ 
actually take in the real world.  That is, if $U=4$ and the value of $Y$
were forced to be 4 (regardless of its actual value), then the value of
$X$ would be 8.

While the equations for a given problem are typically obvious, the
choice of variables may not be.  Consider the following example
(due to Hall \citeyear{Hall98}), showing that the choice of variables influences the causal analysis.
Suppose that
Suzy and Billy both pick up rocks and throw them at  a bottle.
Suzy's rock gets there first, shattering the
bottle.  Since both throws are perfectly accurate, Billy's would have
shattered the bottle had Suzy not thrown.

In this case, a naive model might have
an exogenous variable $U$ that encapsulates whatever background factors
cause Suzy and Billy to decide to throw the rock
(the details of $U$
do not matter, since we are interested only in the context where $U$'s
value is such that both Suzy and Billy throw), a variable $\ST$ for Suzy
throws ($\ST = 1$ if Suzy throws, and $ST = 0$ if she doesn't), a
variable $\BT$ for Billy throws, and a variable $\BS$ for bottle shatters.
In the naive model, whose graph is given in Figure~\ref{fig0}, 
$BS$ is 1 if one of $\ST$ and $\BT$ is 1.
%
\begin{figure}[h]
\begin{center}
\setlength{\unitlength}{.07in}
\begin{picture}(8,8)
\put(3,0){\circle*{.2}}
\put(3,8){\circle*{.2}}
\put(0,4){\circle*{.2}}
\put(6,4){\circle*{.2}}
\put(3,8){\vector(3,-4){3}}
\put(3,8){\vector(-3,-4){3}}
\put(0,4){\vector(3,-4){3}}
\put(6,4){\vector(-3,-4){3}}
\put(3.7,-.2){$\BS$}
\put(-2,3.8){$\ST$}
\put(6.15,3.8){$\BT$}
\put(3.4,7.8){$U$}
\end{picture}
\end{center}
\caption{A naive model for the rock-throwing example.}\label{fig0}
\end{figure}
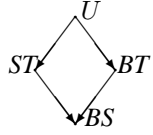

This causal model does not distinguish between Suzy and Billy's rocks 
hitting the bottle simultaneously and Suzy's rock hitting first.
A more sophisticated model might also include variables $\SH$ and $\BH$,
for Suzy's rock hits the bottle and Billy's rock hits the bottle.  
Clearly $\BS$ is 1 iff one of $\SH$ and $\BH$ is 1.  However, now,
$\SH$ is 1 if $\ST$ is 1, and $\BH = 1$ if $\BT = 1$ and $\SH = 0$.
Thus, Billy's
throw hits if Billy throws {\em and\/} Suzy's rock doesn't hit.
This model is described by the following graph,
where we implicitly assume a context where Suzy throws first, so there
is an edge from $\SH$ to $\BH$, but not one in the other direction
(and omit the exogenous variable).
\begin{figure}[h]
{\begin{center}
\setlength{\unitlength}{.07in}
\begin{picture}(8,9)
\put(3,0){\circle*{.2}}
\put(0,8){\circle*{.2}}
\put(6,8){\circle*{.2}}
\put(0,4){\circle*{.2}}
\put(6,4){\circle*{.2}}
\put(0,8){\vector(0,-1){4}}
\put(6,8){\vector(0,-1){4}}
\put(0,4){\vector(1,0){6}}
\put(0,4){\vector(3,-4){3}}
\put(6,4){\vector(-3,-4){3}}
\put(3.8,-.2){$\BS$}
\put(-2.7,7.8){$\ST$}
\put(6.3,7.8){$\BT$}
\put(-2.7,3.8){$\SH$}
\put(6.3,3.8){$\BH$}
\end{picture}
\end{center}
}
\caption{A better model for the rock-throwing example.}\label{fig1}
\end{figure}
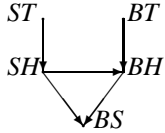

Given a causal model $M = (\cS,\cF)$, a (possibly
empty)  vector
$\vec{X}$ of variables in $\V$, and a vector $\vec{x}$ 
of values for the variables in
$\vec{X}$, we define a new causal model,
denoted $M_{\vec{X} \gets \vec{x}}$,
which is identical to $M$, except that the
equation for the variables $\vec{X}$ in $\cF$ is replaced by $\vec{X} =
\vec{x}$. 
Intuitively, this is the causal model that results when the variables in
$\vec{X}$ are set to $\vec{x}$ by some external action
that affects only the variables in $\vec{X}$
(and overrides the effects of the causal equations).
For example, if $M$ is the more sophisticated model for the
rock-throwing example, 
then $M_{{\it ST} \gets 0}$ is the model where Suzy doesn't throw.

Given a signature $\cS = (\U,\V,\R)$, a formula of the form $X = x$, for
$X \in \V$ and $x \in \R(X)$, is called a {\em primitive event}.   A {\em
basic causal formula\/} 
has the form
$[Y_1 \gets y_1, \ldots, Y_k \gets y_k] \phi$,
where 
\fullv{
\begin{itemize}
\item }
$\phi$ is a Boolean
combination of primitive events;
\fullv{\item}
 $Y_1,\ldots, Y_k$ are distinct variables in $\V$;  and
\fullv{\item }
$y_i \in \R(Y_i)$.
\fullv{\end{itemize}}
Such a formula is abbreviated as $[\vec{Y} \gets \vec{y}]\phi$.
The special
case where $k=0$
is abbreviated as $\phi$.
Intuitively, $[Y_1 \gets y_1, \ldots, Y_k \gets y_k] \phi$ says that
$\phi$ holds in the counterfactual world that would arise if
$Y_i$ is set to $y_i$, for $i = 1,\ldots,k$.
A {\em causal formula\/} is a Boolean combination of basic causal
formulas.

A causal formula $\phi$ is true or false in a causal model, given a
\emph{context}.
We write $(M,\vec{u}) \models \phi$ if
$\phi$ is true in
causal model $M$ given context $\vec{u}$.
$(M,\vec{u}) \models  [\vec{Y} \gets \vec{y}](X = x)$ if 
the variable $X$ has value $x$ 
in the unique (since we are dealing with recursive models) solution
to the equations in
$M_{\vec{Y} \gets \vec{y}}$ in context $\vec{u}$ (i.e., the
unique vector
of values for the exogenous variables that simultaneously satisfies all
equations $F^{\vec{Y} \gets \vec{y}}_Z$, $Z \in \V - \vec{Y}$,
with the variables in $\U$ set to $\vec{u}$).
We extend the definition to arbitrary causal formulas
in the obvious way.

\subsection{Causality}

We now review the updated HP definition of causality.
\begin{definition}\label{def-cause}
$\vec{X} = \vec{x}$ is a {\em cause\/} of $\varphi$ in
$(M,\vec{u})$ if the following three conditions hold: 
\begin{description}
\item[AC1.] $(M,\vec{u}) \models (\vec{X} = \vec{x}) \wedge \varphi$. 
\item[AC2.] There exist a partition $(\vec{Z},\vec{W})$ of $\V$ with 
$\vec{X} \subseteq \vec{Z}$ and some setting 
$(\vec{x}',\vec{w})$ of the
variables in $(\vec{X},\vec{W})$ such that if $(M,\vec{u}) \models Z = z^*$
for $Z \in \vec{Z}$, then
\be
\item[(a)] $(M,\vec{u}) \models [ \vec{X} \leftarrow \vec{x}',
\vec{W} \leftarrow \vec{w}]\neg{\varphi}$. 
\item[(b)] $(M,\vec{u}) \models [ \vec{X} \leftarrow \vec{x},
\vec{W}' \leftarrow \vec{w}, \vec{Z}' \leftarrow \vec{z}^*]\varphi$ for
all subsets $\vec{Z}'$ of $\vec{Z} \setminus \vec{X}$ and all subsets 
$\vec{W}'$ of $\vec{W}$, 
where we abuse notation and write $\vec{W}' \gets \vec{w}$ to
denote the assignment where the variables in $\vec{W}'$ get the same
values as they would in the assignment $\vec{W} \gets \vec{w}$, and
similarly for $\vec{Z}' \gets \vec{z}^*$. 
That is, setting any subset $\vec{W}'$ of $\vec{W}$ to the values in $\vec{w}$
should have no effect on $\varphi$ as long as $\vec{X}$ has the value 
$\vec{x}$, even if all the variables in an arbitrary subset of $\vec{Z}$
are set to their original values in the context $\vec{u}$.  
\fullv{
The tuple $(\vec{W}, \vec{w}, \vec{x}')$ is said to be a
\emph{witness} to the fact that $\vec{X} = \vec{x}$ is a cause of $\phi$.}
\ee
\item[AC3.] $(\vec{X} = \vec{x})$ is minimal; no subset of
$\vec{X}$ satisfies AC2.
\end{description}
\end{definition}
If $\vec{X}$ is a singleton, then $X=x$ is said to be a \emph{singleton
cause} of $\phi$ in $(M,\vec{u})$.

AC1 just says that $A$ cannot be a cause of $B$ unless both $A$ and $B$
are true.
The core of this definition lies in AC2.
Informally, the variables in $\vec{Z}$ should be thought of as
describing the ``active causal process'' from $X$ to $\phi$.
These are the variables that mediate between $X$ and $\phi$.
AC2(a) is reminiscent of the traditional counterfactual criterion,
according to which $X=x$ is a cause of $\phi$ if changing
the value of $X$ results in $\phi$ being false.
However, AC2(a) is more permissive than the traditional criterion;
it allows the dependence of $\phi$ on $X$ to be tested
under special {\em structural contingencies}, 
in which the variables $\vec{W}$ are held constant at some setting
$\vec{w}$.  AC2(b) is an attempt to counteract the ``permissiveness''
of AC2(a) with 
regard to structural contingencies.  Essentially, it ensures that
$X$ alone suffices to bring about the change from $\phi$ to $\neg\phi$; 
setting $\vec{W}$ to $\vec{w}$ merely eliminates
spurious side effects that tend to mask the action of $X$. 

To understand the role of AC2(b), consider the rock-throwing example
 again.  Let $M$ be the model in Figure~\ref{fig0}, and let $\vec{u}$ be
the context where both Suzy and Billy throw.  It is easy to see that 
both Suzy and Billy are causes of the bottle shattering in
$(M,\vec{u})$: Let $\vec{Z} = \{\ST,BS\}$, and
consider the structural contingency where Billy doesn't throw ($\BT = 
0$).  Clearly $(M,U) \models [\ST \gets 0, \BT \gets 0](\BS = 0)$ and
$(M,u) \models [\ST  \gets 1, \BT   \gets 0](BS = 1)$, so Suzy is a
cause of the bottle 
shattering. A symmetric argument shows that Billy is also a cause. 

But now consider the model 
$M'$ described in Figure~\ref{fig1}; again, $u$ is the context where
both Suzy and Billy throw.  It is still
the case that Suzy is a cause of the bottle shattering in $(M',u)$.  We
can take $\vec{W} = \{\BT\}$ and again consider the contingency where
Billy doesn't 
throw. 
However, Billy is {\em not\/} a cause of the bottle shattering in
$(M',u)$.  For 
suppose that
we now take $\vec{W} = \{\ST\}$ and consider the contingency
where Suzy doesn't throw.  Clearly AC2(a) holds, since if Billy doesn't
throw (under this contingency), then the bottle doesn't shatter.
However, AC2(b) does not hold.  Since $\BH \in \vec{Z}$, if we set $\BH$ to
0 (its original value), then AC2(b) would require that 
$(M',u) \models [\BT \gets 1, \ST \gets 0, \BH \gets 0](BS = 1)$, but
this is not the case.
Similar arguments show that no other choice of $(\vec{Z},\vec{W})$ makes
Billy's throw a cause of the bottle shattering in $(M',u)$.

The original HP definition differs from the updated definition in only
one respect.  
Rather than requiring that
$(M,\vec{u}) \models [\vec{X} \gets
\vec{x}, \vec{W}' \gets \vec{w}, \vec{Z}' \gets \vec{z}^*]\phi$ 
for all subsets $\vec{W}'$ of $\vec{W}$,
it was required to hold only for $\vec{W}$.  That is, the following
condition was used 
instead of AC2(b).  
\begin{description}
\item[AC2(b$'$)] 
$(M,\vec{u}) \models [ \vec{X} \gets
\vec{x}, \vec{W} \gets \vec{w}, \vec{Z}' \gets \vec{z}^*]\phi$ for 
all subsets $\vec{Z'}$ of
$\vec{Z}$. 
\end{description}

The requirement for AC2(b) to hold for all subsets of $W$ 
in the updated definition prevents situations 
where $W$ ``conceals other causes for $\varphi$''. 
The role of this requirement is perhaps best understood by considering
the following example, due to Hopkins and Pearl \citeyear{HopkinsP02}
(the description is taken from \cite{HP01b}):
Suppose that a prisoner dies 
either if $A$ loads $B$'s gun and $B$ shoots, or if $C$ loads and shoots
his gun.  Taking $D$ to represent the prisoner's death and making the
obvious assumptions about the meaning of the variables, we have that
$D= (A\land B) \lor C$.  Suppose that in the actual
context $u$, $A$ loads $B$'s gun, $B$ does not shoot, but $C$ does load
and shoot his gun, so that the prisoner dies.  That is, $A=1$, $B=0$,
and $C=1$. Clearly $C=1$ is a cause of $D=1$. 
We would not want to say that $A=1$ is a cause of $D=1$,
given that $B$ did not shoot (i.e., given that $B=0$).  However,
with AC2(b$'$), $A=1$ is a cause of $D=1$.  For we can take
$\vec{W} = \{B,C\}$ and consider the contingency where $B=1$ and $C=0$.
It is easy to check that AC2(a) and AC2(b$'$) hold for this contingency,
so under the original HP definition, $A=1$ is a cause of $D=1$.  However,
AC2(b) fails in this case, since $(M,u) \models [ A \gets 1, C \gets 0](D=0)$.
The key point is that AC2(b) says that for $A=1$ to be a cause of $D=1$,
it must 
be the case that $D=0$ if only some of the values in $\vec{W}$ are set
to $\vec{w}$.  That means that the other variables 
get the same value as they do in the actual context; in this case, by
setting only $A$ to 1 and leaving $B$ unset, $B$ takes on its original
value of 0, in which case $D=0$.  AC2(b$'$) does not consider this case.

Using AC2(b) rather than AC2(b$'$) has been shown to have a significant
benefit
(and to lead to more intuitive results) when causality is applied to
program verification, with the goal of understanding what in the code is
the cause of a program not satisfying its specification \cite{BBCOT12}.

\section{Relevant Complexity Classes}\label{sec:complexity}
In this section, we briefly recall the definitions of the complexity
classes that we need for our results, and define the complexity class
$D_2^k$.  

Recall that the \emph{polynomial hierarchy} is a hierarchy of complexity
classes that generalize the classes $\NP$ and co-$\NP$.  Let $\Sigma^P_1 =
\NP$ and $\Pi^P_1 = \mbox{co-}\NP$.  For $i > 1$, define $\Sigma^P_i =
\NP^{\Sigma_{i-1}^P}$ and $\Pi_i^P = (\mbox{co-}\NP)^{\Sigma_{i-1}^P}$, where, in
general, $X^Y$ denotes the class of problems solvable by a Turing
machine in class $A$ augmented with an oracle for a problem complete for
class $B$.  (See \cite{MS72,Stock} for more details and intuition.)


We now define the classes $D^P_k$ as follows. 
 \begin{definition}\label{def-dk}
For  $k = 1, 2, \ldots$,
\[D^P_k = \{ \Lan: \exists \Lan_1, \Lan_2: \Lan_1 \in \Sigma^P_k, \Lan_2
\in \Pi^P_k,  
\Lan = \Lan_1 \cap \Lan_2 \}. \]
\end{definition}
For $k=1$, the class $D^P_1$ is the well-known complexity class $\DP$,
defined by Papadimitriou and Yannakakis \citeyear{PY84}. It contains
\emph{exact} problems such as the language of pairs $\zug{G,k}$, 
where $G$ is a graph that has a maximal clique of size exactly $k$.
As usual, we say that a language $\Lan$ is $D^P_k$ complete if it is in
$D^P_k$ and is the ``hardest'' language in $D^P_k$, in the sense that
there is a polynomial time reduction from any language $\Lan' \in D^P_k$
to $\Lan$. 

Recall that a {\emph quantified Boolean formula} (QBF) is a generalization of
a propositional formula, where some propositional variables are
quantified.   Thus, for example, $\exists x \forall y (x \lor y)$ is a
QBF.  A \emph{closed} QBF (CQBF) is one where there are no free
propositional variables.  A CQBF is either true or
false, independent of the truth assignment.  The ``canonical'' languages
complete for $\Sigma_2^k$ and $\Pi_2^k$ consist of the 
CQBFs with $k$ alternations of quantifiers starting with
$\exists$ (resp., $\forall$) that are true.  In particular, let
$$\begin{array}{ll}
\Sigma_{2}^P(\mbox{SAT})=\\
\{ \exists\vec{X}\forall\vec{Y}\varphi \mid 
\exists\vec{X}\forall\vec{Y}\varphi \mbox{ is a CQBF, }
\exists\vec{X}\forall\vec{Y}\varphi = \T\}\\
\Pi_{2}^P(\mbox{SAT})=\\
\{ \forall\vec{X}\exists\vec{Y}\varphi \mid 
\forall\vec{X}\exists\vec{Y}\varphi \mbox{ is a CQBF, }
\forall\vec{X}\exists\vec{Y}\varphi =\T\}.
\end{array}
$$
$\Sigma_2^P(\mbox{SAT})$ is complete for $\Sigma_2^P$ and 
$\Pi_2^P(\mbox{SAT})$ is complete for $\Pi_2^P$~\cite{W76}.

The following lemma provides a useful condition sufficient for a language
to be $D^P_k$-complete.
\begin{lemma}\label{lemma-dk}
If $\Lan_1$ is $\Sigma^P_k$-complete and $\Lan_2$ is $\Pi^P_k$-complete, then
$\Lan_3 = \Lan_1 \cap \Lan_2$ is $D^P_k$-complete. 
\end{lemma}
\begin{proof}
The fact that $\Lan_3$ is in $D^P_k$ is immediate from the definition of 
$D^P_k$. For hardness, let $\Lan'_3$ be a language in $D^P_k$. Then
there exist $\Lan'_1$ and $\Lan'_2$ such that 
$\Lan'_1 \in \Sigma^P_k, \Lan'_2 \in \Pi^P_k$, and $\Lan' = \Lan'_1 \cap
\Lan'_2$. Let $f$ be a polynomial-time reduction from $\Lan'_1$ to $\Lan_1$,
and let $g$ be a polynomial-time reduction from $\Lan'_2$ to $\Lan_2$ (the
existence of such reductions $f$ and $g$ follows from the fact that
$\Lan_1$ and $\Lan_2$ are $\Sigma^P_k$-complete and $\Pi^P_k$-complete,
respectively). Then, $\zug{f,g}$ 
is a polynomial-time reduction from $\Lan'_3$ to $\Lan_3$, as required. 
\end{proof}
Essentially the same argument shows that if $\Lan_1$ is
$\Sigma^P_k$-hard and $\Lan_2$ is $\Pi_k^P$-hard, then $\Lan_3 = \Lan_1 \cap \Lan_2$
is $D^P_k$-hard.

Determining whether $\vec{X} = \vec{x}$ is 
a cause of $\phi$ in $(M,u)$ is
a decision problem:
we define a language and try to determine whether a particular tuple is
in that language.  (See Section~\ref{sec-cause} for the formal
definition.)  Determining degree of responsibility and blame is a
different type 
of problem, since we are determining which number represents the
degree of
responsibility (resp., blame).  Formally, these are \emph{function
  problems}.   
For ease of exposition, we restrict attention to functions from some
strings over some fixed language $\Sigma$ to strings over $\Sigma$
(i.e., we are considering functions from $\Sigma^*$ to $\Sigma^*$).  
For a complexity class $A$ in the polynomial hierarchy, \fpa\ consists of all 
functions that can be computed 
by a polynomial-time Turing machine with an $A$-oracle 
which on input $x$ asks a total of $O(\log{|x|})$ queries \cite{Pap84}.
A function $f(x)$ is \fpa\-hard iff for every function $g(x)$
in \fpa \ there exist polynomially computable functions
$R,S: \Sigma^* \rightarrow \Sigma^*$ 
such that $g(x) = S(f(R(x)))$. A function $f(x)$ is complete in \fpa \ iff
it is in \fpa \ and is \fpa-hard.

Finally, for a complexity class $A$ in polynomial hierarchy,  \fpapar
is the class of  
functions that can be computed by a polynomial-time Turing machine with 
parallel (i.e., non-adaptive) queries to an $A$-oracle. (For background
on these complexity classes, see  \cite{JT95,Joh90}.)  

\section{Complexity for the Updated HP Definition}\label{sec-cause}
In this section, we prove our results on the complexity of deciding
causality.  We start by defining the problem formally.
In the definitions, $M$ stands for a causal model, $\vec{u}$ is a context,   
$\vec{X}$ is a subset of variables of  $M$, and $\vec{x}$ is the set of
values of $\vec{X}$ in $(M, \vec{u})$: 
$$\begin{array}{ll}
 \Lc = &\{ \pair{M,\vec{u},\varphi, \vec{X}, \vec{x}} : (\vec{X}=\vec{x}) \\
 & \ \ \ \mbox{ is a cause of } \varphi \mbox{ in } (M,\vec{u}) \}.
\end{array}$$
One of our goals is to understand the cause of the complexity of
computing causality.  Towards this end, it is useful to define two
related languages:
$$\begin{array}{ll}
 \Lsingle = &\{ \pair{M,\vec{u},\varphi, \vec{X}, \vec{x}} :
 (\vec{X}=\vec{x}) \mbox{ satisfies conditions}\\ 
&\ \ \  \mbox{AC1 and AC2 of Def.~\ref{def-cause} for } \varphi \mbox{ in }
 (M,\vec{u}) \}, \\
 \Lmin = &\{ \pair{M,\vec{u},\varphi, \vec{X}, \vec{x}} : (\vec{X}=\vec{x}) \mbox{ satisfies conditions}\\
 & \ \ \ \mbox{AC1 and AC3 of Def.~\ref{def-cause} for } \varphi \mbox{ in } (M,\vec{u}) \}.
\end{array}
$$
%
It is easy to see that $\Lc = \Lsingle \cap \Lmin$. 

Let $\Lco$ be the subset of $\Lc$ where $\vec{X}$ and $\vec{x}$ are
singletons; this is the singleton causality problem.  We can similarly
define $\Lsingleo$ and $\Lmino$.  Again, we have $\Lco = \Lsingleo \cap
\Lmino$, but, in fact, we have $\Lco = \Lsingleo$, since $\Lsingleo 
\subseteq \Lmino$; for singleton causality, the
minimality condition AC3 trivially holds.

We denote by $\Lcbin$ the language of causality for binary causal
models
(i.e., where the models $M$ in the tuple are binary models),
and by $\Lsinglebin$ and $\Lminbin$ the languages $\Lsingle$ and
$\Lmin$ restricted to binary causal models. Again we have that
$\Lcbin = \Lsinglebin \cap \Lminbin$.  
And again, we can define 
$\Lcbino$,  $\Lsinglebino$, and $\Lminbino$, and we have 
$\Lcbino = \Lsinglebino$.

We start by considering singleton causality.  As we observed,
Eiter and Lukasiewicz \citeyear{EL01} and Hopkins~\citeyear{Hopkins01} showed
that, with the original HP definition, singleton causality and causality
coincide.  
However, for the updated definition, Halpern \citeyear{Hal39} showed that
it is in fact possible to have minimal causes that are not
singletons. 
Thus, we consider singleton causality and general causality separately.
We can clarify where the complexity lies by considering $\Lsingle$ (and
its sublanguages) and  $\Lmin$ (and its sublanguages) separately.

\begin{theorem}\label{theorem-lsinglebin-complete}
The languages $\Lsingle$, $\Lsingleo$, $\Lsinglebino$, and $\Lsingleo$ are
$\Sigma^P_2$-complete. 
\end{theorem}
\begin{proofout}
To show  all these languages are in $\Sigma^P$, 
given a tuple $\zug{M,\vec{u},\phi,\vec{X},\vec{x}}$, checking that
AC1 holds, that is, checking that $(M,\vec{u}) \models \vec{X} = \vec{x}
\land \phi$, can be done in time polynomial in the size of $M$,
$|\vec{X}|$, and $|\phi|$ (the length of $\phi$ as a string of symbols).
For AC2, 
we need only  guess the set $\vec{W}$ and the assignment
$\vec{w}$. The check that  
assigning $\vec{w}$ to $\vec{W}$ and $x'$ to $X$ indeed falsifies
$\varphi$ is polynomial, 
and we use an $\NP$ oracle to check that for all subsets of $\vec{W}$ and
all subsets of $\vec{Z}$, condition AC2(b) holds.  
(The argument is quite similar to Eiter and Lukasiewicz's argument that
causality is in $\Sigma_2^P$ for general models with the original HP
definition.) 

%
%
For hardness, it clearly suffices to show that
$\Lsinglebino$ is $\Sigma_2^P$-hard.  We do this by reducing 
$\Sigma_2^P(\mbox{SAT})$ to $\Lsinglebino$.  Given a CQBF formula 
$\exists\vec{X}\forall\vec{Y}\varphi$, we show that we can efficiently
construct a causal formula $\psi$, model $M$, and context $u$ such that 
$\exists\vec{X}\forall\vec{Y}\varphi = \T$ iff 
$(M,u,\psi,A,0) \in \Lsinglebino$.
\shortv{We leave details to the full paper.}
\fullv{We leave details to the appendix.}
\end{proofout}

Since, as we have observed, AC3 is vacuous in the case of singleton
causality, it follows that singleton causality is $\Sigma_2^P$-complete.
\begin{corollary}\label{cor-strong-cause-binary}
$\Lco$ and $\Lcbino$ are $\Sigma_2$-complete.
\end{corollary}

\commentout{
It now remains to show that the same result holds for general causal models.
\begin{theorem}\label{theorem-strong-cause-general}
Singleton causality is $\Sigma_2$-complete for general causal models.
\end{theorem}
\begin{proof}
Hardness in $\Sigma_2$ follows from Corollary~\ref{cor-strong-cause-binary}.
For membership, note that
we need only guess the set $\vec{W}$ and the assignment
$\vec{w}$. The check that  
assigning $\vec{w}$ to $\vec{W}$ and $x'$ to $X$ indeed falsifies $\varphi$ is polynomial,
and we use an $\NP$ oracle to check that for all subsets of $\vec{W}$ and all subsets of $\vec{Z}$, the condition AC2 holds (the proof is fairly similar to the original one).
\end{proof}

}

We now show that things are harder if we do not restrict to binary causal
models (unless the polynomial hierarchy collapses). As a first step, we
consider the complexity of $\Lmin$ and $\Lminbin$.


\begin{theorem}\label{theorem-pi-2}
$\Lmin$ and $\Lminbin$ are $\Pi^P_2$-complete.
\end{theorem}
\begin{proofout}
The fact that $\Lmin$ and $\Lminbin$ are in $\Pi^P_2$ is straightforward.  
Again, given a tuple $\zug{M,\vec{u},\phi,\vec{X},\vec{x}}$, we can check that
AC1 holds in polynomial time.  For AC3, we need to check that for all
strict subsets $\vec{X}'$ of $\vec{X}$, AC2 fails.  Since checking AC2 is in
$\Sigma_2^P$, checking that it fails is in $\Pi_2^P$.  Checking that it
fails for all strict subsets $\vec{X}'$ keeps it in $\Pi_2^P$ (since it
just adds one more universal quantifier).

To prove that these languages are $\Pi_2^P$-hard, we show that we can
reduce $\Pi_2^P(\mbox{SAT})$ to $\Lminbin$.  
The proof is similar in spirit to the proof of
Theorem~\ref{theorem-lsinglebin-complete};
\shortv{we leave details to the full paper.}
\fullv{we leave details to the appendix.}
%
\end{proofout}

We are now ready to prove our main result.
\begin{theorem}\label{theorem-main}
$\Lc$ and $\Lcbin$ are $D^P_2$-complete.
\end{theorem}
\begin{proof}
Membership of $\Lc$ (and hence also $\Lcbin$) in $D^P_2$ follows from
the fact that $\Lc = \Lsingle \cap \Lmin$, $\Lsingle \in \Sigma^P_2$,
and $\Lmin \in \Pi^P_2$. 
The fact that $\Lcbin$ (and hence also $\Lc$) are $D^P_2$-hard follows from
Lemma~\ref{lemma-dk} and Theorems~\ref{theorem-lsinglebin-complete}
and~\ref{theorem-pi-2}. 
\end{proof}

\fullv{
The fact that there may be more than one conjunct in a cause using the
updated HP definition means that
checking AC3 becomes nontrivial, and causes the increase in complexity
for $\Sigma_2^P$ to $D_2^P$.  But why is there no dropoff with the
updated HP definition when we restrict to binary models, although there
is a dropoff from $\Sigma_2^P$ to $\NP$ for the original HP definition?
To prove their $\NP$-completeness result, Eiter and Lukasiewicz
\citeyear{EL01} showed that for binary models, with the original HP
definition, the set $\vec{Z}$ and its subsets can be omitted from the 
definition of cause.  That is, we can replace AC2(b$'$) by 
\begin{description}
\item[AC2(b$''$)] 
$(M,\vec{u}) \models [ \vec{X} \gets \vec{x}, \vec{W} \gets \vec{w}]\phi$
\end{description}
to get an equivalent definition.
The example that a cause may require more than one conjunct given by
Halpern \citeyear{Hal39} shows that removing $\vec{Z}$ and its subsets
from AC2(b) does not result in an equivalent definition in binary models.
But even if it did, the fact that we need
to quantify over all subset $\vec{W}'$ of $\vec{W}$ in AC2(b) would be
enough to ensure that there is no dropoff in complexity in binary models.
}

\section{Responsibility and Blame}

In this section, we review the definitions of responsibility and blame
and characterize their complexity.
See Chockler and Halpern \citeyear{CH04} for more intuition and details.

\subsection{Responsibility}
The definition of responsibility given by Chockler and Halpern
\citeyear{CH04} was given based on the original HP definition of
causality, and thus assumed that causes were always single conjuncts.
It is straightforward to 
extend it to allow causes to have arbitrarily many conjuncts.
\begin{definition}\label{def-resp}
The {\em degree of responsibility
of $\vec{X}=\vec{x}$ for $\phi$ in 
$(M,\vec{u})$\/}, denoted $\dr((M,\vec{u}), (\vec{X}=\vec{x}), \phi)$, is
$0$ if $\vec{X}=\vec{x}$ is 
not a cause of $\phi$ in $(M,\vec{u})$; it is $1/(k+1)$ if
$\vec{X}=\vec{x}$ is  a cause of $\phi$ in $(M,\vec{u})$ 
and there exists a partition $(\vec{Z},\vec{W})$ and setting
$(\vec{x}',\vec{w})$ for which AC2 holds
such that (a) $k$ variables in $\vec{W}$ have different values in $\vec{w}$
than they do in the context $\vec{u}$ and (b) there is no partition 
$(\vec{Z}',\vec{W}')$ and setting $(\vec{x}'',\vec{w}')$ satisfying AC2
such that only $k' < k$ variables have different values in $\vec{w}'$
than they do the context $\vec{u}$.
\end{definition}

Intuitively, $\dr((M,\vec{u}), (\vec{X}=\vec{x}), \phi)$ measures the minimal number 
of changes that have to be made in $\vec{u}$ 
in order to make $\phi$ counterfactually depend on $\vec{X}$, provided
the conditions on the subsets of $\vec{W}$ and $\vec{Z}$ are satisfied
(see also the voting example from the introduction). 
If there is no partition of $\V$ to
$(\vec{Z},\vec{W})$ that satisfies AC2, or $(\vec{X}=\vec{x})$ does not satisfy AC3 for $\phi$
in $(M,\vec{u})$,  then the minimal number of changes in $\vec{u}$ 
in Definition~\ref{def-resp} is taken to have cardinality
$\infty$, and thus the degree of responsibility of $(\vec{X}=\vec{x})$ is $0$ (and hence it is not a cause). 
%


In the original HP model, it was shown that computing responsibility is
\fp-complete in binary causal models  \cite{CHK} and
\fpsigma-complete in general causal models \cite{CH04}.  We now
characterize the complexity of computing responsibility in the updated
HP definition.

\begin{theorem}\label{theor-strong-resp}
Computing the degree of responsibility is \fpsigma-complete for
singleton causes in binary and general causal models.   
\end{theorem}
\begin{proofout} 
The proof is quite similar to the proof in \cite{CH04}.
We prove membership by describing an algorithm 
in \fpsigma for computing the degree of responsibility. Roughly speaking, the algorithm queries an oracle for the language ${\cal R} = \{( \zug{(M,\vec{u}), (X=x), \varphi,i}$ such that
$\zug{(M,\vec{u}), (X=x), \varphi} \in \Lan_{cause}$ and the degree of responsibility of $(X=x)$ for $\varphi$ is at least $i \}$. It is easy to see that ${\cal R}$ is in $\Sigma^P_2$ by using 
Corollary~\ref{cor-strong-cause-binary}.
The algorithm for computing the degree of responsibility performs a 
binary search on the 
value of $\dr((M,\vec{u}), (X=x), \varphi)$, each time
dividing the range of possible values for the degree of responsibility 
by $2$ according to the answer of ${\cal R}$. 
The number of possible candidates for the degree of responsibility is bounded by the size of the input $n$, and thus the number of queries is at most $\lceil \log{n} \rceil$. 

For hardness in binary causal models (which implies hardness in general
causal models), we provide a reduction from the 
$\Sigma_2^P$-complete problem
MINQSAT$_2$ \cite{CH04} to the
degree of responsibility, where MINQSAT$_2(\exists \vec{X} \forall
\vec{Y} \psi)$ is the minimum number of $1$'s in the satisfying
assignment to $\vec{X}$ for $\exists \vec{X} \forall \vec{Y} \psi$ if
such an assignment exists, and 
$|\vec{X}| + 1$ otherwise.

\end{proofout}

\begin{theorem}\label{theor-actual-resp}
Computing the degree of responsibility is \fpsigmad-complete in binary and general causal models.  
\end{theorem}
\fullv{
\begin{proofout} 
Membership in \fpsigmad is shown in quite a similar way to
Theorem~\ref{theor-strong-resp}. For hardness, as there are no known
natural problems complete in \fpsigmad, the proof 
proceeds by constructing a generic reduction from a problem in \fpsigmad\ to
the degree of 
responsibility.
\end{proofout}
}

\subsection{Blame}

The definition of blame addresses the situation where there is 
uncertainty about the true situation or ``how the world works''. 
Blame, introduced in \cite{CH04}, considers the 
``true situation'' to be determined by the context, and
``how the world works'' to be determined by the structural equations.
An agent's uncertainty is modeled by a pair $(\K,\Pr)$, where $\K$
is a set of pairs of the form $(M,\vec{u})$,
where $M$ is a causal model and $\vec{u}$ is a context,
and $\Pr$ is a probability distribution over $\K$.
A pair $(M,\vec{u})$ is called a {\em situation}.  
We think of $\K$ as describing the situations that the agent considers
possible before $\vec{X}$ is set to $\vec{x}$.
The degree 
of blame that setting $\vec{X}$ to $\vec{x}$ has for $\phi$ is then the expected degree
of responsibility of $\vec{X}=\vec{x}$ for $\phi$ in  $(M_{\vec{X} \gets \vec{x}},\vec{u})$,   
taken over the situations $(M,\vec{u}) \in \K$. 
Note that the situation $(M_{\vec{X} \gets \vec{x}},\vec{u})$ for $(M,
\vec{u}) \in \K$ are those that the agent considers possible after $\vec{X}$ is set to
$\vec{x}$. 

\begin{definition}\label{def:blame}
The {\em degree of 
blame of setting $\vec{X}$ to $\vec{x}$ for $\phi$ relative to epistemic state
$(\K,\Pr)$\/}, denoted $\db(\K,\Pr,\vec{X} \gets \vec{x}, \phi)$, is
$$\sum_{(M,\vec{u}) \in \K}
\dr((M_{\vec{X} \gets \vec{x}}, \vec{u}), \vec{X} = \vec{x}, \phi)
\Pr((M,\vec{u})).$$ 
\end{definition} 

For the original HP definition of cause, Chockler and Halpern
\citeyear{CH04} show that 
computing the degree of blame is complete in  
\fpsigmapars  for general and in \fppar for binary causal models.
Again, with the updated HP definition, the complexity changes.
\begin{theorem}\label{theor-complexity-blame}
The problem of computing blame in recursive causal models is
\fpsigmapar-complete for singleton causes and \fpsigmapard-complete for (general) causes, in binary and general causal models.  
\end{theorem}

\fullv{
\appendix


\section{Proof of Theorem~\ref{theorem-lsinglebin-complete}}

As we oberved in the main part of the paper, membership is
straightforward, so we focus here on hardness.
For hardness, we describe a reduction from the language 
$\Sigma_{2}(\mbox{SAT})$ to $\Lsinglebino$.
In the process, we work with both propositional formulas with
propositional variables, and causal formulas, that use formulas like $X
= 1$ and $X=0$.  We can think of $X$ as a propositional variable here,
where $X=1$ denotes that $X$ is true, and $X=0$ denotes that $x$ is false.
If $\phi$ is a propositional formula, let
$\overline{\phi}$ be the causal formula that results by replacing each
occurrence of a propositional variable $X$ by $X=1$.

Given a CQBF $\exists\vec{X}\forall\vec{Y}\varphi$, 
consider the tuple 
$(M,u,\psi,A,0)$ where $M = (\U,\V,\R)$ is a binary causal model and
\begin{itemize}
\item $\U= \{U\} $;
\item $\V= \{ X^{0}\ |\ X\in\vec{X}\} \cup \{ X^{1}\
|\ X\in\vec{X}\} \cup \{ Y\ |\ Y\in\vec{Y}\} \cup \{A\} $,
where $A$ is a fresh variable that does not appear in $\vec{X}$ or
$\vec{Y}$;
\item for all variables $V\in \vec{\V}$, the structural
equation is  $V=U$
(i.e. all the variables in $\mathcal{V}$ are set to the value of
$U$);
\item $u=0$;
\item $\psi=\psi_{1}\vee (\psi_{2}\wedge\psi_{3})$ where
$\psi_{1},\psi_{2},\psi_{3}$ 
are the following causal formulas:
\begin{itemize}
\item $\psi_{1}=\neg\left(\bigwedge_{X\in\vec{X}}(X^{0}\ne X^{1})\right)$;%
\footnote{As usual, we take $X^0 \ne X^1$ to be an abbreviation for the
causal formula $(X^0 = 1 \land  X^1 = 0) \lor (X^0 = 0 \land X^1 = 1)$.}
\item $\psi_{2}=\neg(A=1\wedge \vec{Y}=\vec{1})$; 
\item $\psi_{3}= A=1 \vee\overline{\varphi}[\vec{X}/\vec{X}^1]$, where
$\overline{\varphi}[\vec{X}/\vec{X}^1]$ is the result of replacing each
occurrence of a variable $X \in \vec{X}$ by $X^1$).
\end{itemize}
\end{itemize}
We prove that 
$\exists\vec{X}\forall\vec{Y}\varphi = \T$ iff $A=0$ is 
a cause of $\psi$ in $(M,u)$ (which is the case iff 
$(M,u,\psi,A,0) \in \Lsinglebino$, since AC3 is vacuous for binary models).

First suppose that $\exists\vec{X}\forall\vec{Y}\varphi = \T$.  To
show that $A=0$ is a cause of $\psi$ in $(M,u)$, we prove that AC1 and
AC2 hold.

%

Clearly AC1 holds: $(M,u) \models A=0$ by the definition of $F_A$,
and $(M,u) \models \psi$ since $(M,u) \models \psi_1$, again by the
definition of $\cF$.

For AC2, let $\vec{W} = \V - \{A\}$. and define $\vec{w}$ as follows.  
Let $\tau$ be an assignment to the variables in $\vec{X}$ for which
$\forall\vec{Y}\varphi=\T$. Using $\vec{w}(X)$ to denote the value of
$X$ according to $\vec{w}$, we require that
\begin{itemize}
\item $\vec{w}(X^{\tau(X)})=1$;
\item $\vec{w}(X^{1-\tau(X)})=0$; and 
\item $\vec{w}(Y)=1$. 
\end{itemize}
For AC2(a), 
note that
$(M,u)\models [A\leftarrow 1, \vec{W}\leftarrow\vec{w}]\neg
\psi_{1}$ 
(since $\vec{w}$ assigns different values to $X^{0}$ and $X^{1}$ for 
all $X\in\vec{X}$) and, 
since $\vec{w}(Y)=1$ for all
$Y\in\vec{Y}$, 
we have that $(M,u) \models
[A\leftarrow 1, \vec{W}\leftarrow\vec{w}]\neg \psi_{2}$, 
so $(M,u)\models
[\vec{A\leftarrow 1, W}\leftarrow\vec{w}]\neg \psi$.  Thus,
AC2(a) holds.

It now remains to show that AC2(b) holds. 
Fix $\vec{W}^{\prime}\subseteq\vec{W}$.
We must show that
$(M,u)\models [A \leftarrow 0,
\vec{W}^{\prime}\leftarrow\vec{w}]\psi$.  
(The condition ``for all $\vec{Z}' \subseteq \vec{Z}-\{A\}$'' is vacuous
in this case, since $\vec{Z} = \{A\}$.)  Since 
the definition of $M$ guarantees that $(M,u)\models [A \leftarrow 0,
\vec{W}^{\prime}\leftarrow\vec{w}]\psi$ iff 
$(M,u)\models
[\vec{W}^{\prime}\leftarrow\vec{w}]\psi$,  we focus on
the latter from here on in.

If $(M,u)\models[\vec{W}^{\prime}\leftarrow\vec{w}]\psi_{1}$ ,
we are done.  So suppose that
$(M,u)\models[\vec{W}^{\prime}\leftarrow\vec{w}]\neg \psi_1$; that is, 
\begin{equation}\label{eq1}
(M,u)\models[\vec{W}^{\prime}\leftarrow\vec{w}]\left(\bigwedge_{X\in\vec{X}}(X^{0}\ne
X^{1})\right).
\end{equation}
It follows that, for each variable $X \in \vec{X}$, 
we have that $(M,u) \models [\vec{W}' \leftarrow \vec{w}](X^1 = \tau(X))$.
To see this, note that if $\tau(X) = 1$, then we must have $X^1 \in
\vec{W}'$; otherwise, we would have $(M,u) \models [\vec{W}'
\leftarrow \vec{w}](X^1 = 0 \land X^0 = 0)$, contradicting (\ref{eq1}).
And if $\tau(X) = 0$, then since $\vec{w}(X^1) = 0$, we must have 
$(M,u) \models [\vec{W}' \leftarrow \vec{w}](X^1 = 0)$, whether or not
$X^1 \in \vec{W}'$, so  
$(M,u) \models [\vec{W}' \leftarrow \vec{w}]\psi_3$.
It follows that $(M,u) \models [\vec{W}' \leftarrow \vec{w}]\psi$,
showing that AC2(b) holds.

Finally, we must show that 
if $A=0$ is a cause of $\psi$ in $(M,u)$ then
$\exists\vec{X}\forall\vec{Y}\varphi=\T$.  

So suppose that $A=0$ is a cause of $\psi$ in $(M,u)$.  Then there
exists a witness  $(\vec{W},\vec{w},a)$.  Since we are considering
binary models, we must have $a=1$, so we have 
\begin{equation}\label{eq2}
(M,u) \models [A \leftarrow
1,\vec{W}\leftarrow\vec{w}]\neg \psi.
\end{equation} 
This implies that
$(M,u)[A \leftarrow 1, \vec{W}\leftarrow\vec{w}]
\models \neg \psi_{1}$,
so $$(M,u) \models [A \leftarrow 1, \vec{W}\leftarrow\vec{w}]
\left( \bigwedge_{X \in \vec{X}} (X^0 \ne X^1)\right).$$
Define $\tau$ so that 
$\tau(X)=b$, where $b\in\{ 0,1\} $
is the unique value for which
$(M,u)[A \leftarrow 1, \vec{W}\leftarrow\vec{w}] \models
X^b = 1$. 

It also follows from (\ref{eq2}) that 
$$(M,u) \models [A \leftarrow 1,\vec{W}\leftarrow\vec{w}] \neg
(\psi_2 \land \psi_3).$$  
Since clearly $(M,u)\models [A \leftarrow 1,
\vec{W}\leftarrow\vec{w}]\psi_2$, we must have
$(M,u)\models [A \leftarrow 1,
\vec{W}\leftarrow\vec{w}]\neg \psi_3$.
Indeed, we must have 
$$(M,u)\models [A \leftarrow 1,
\vec{W}\leftarrow\vec{w}] (\vec{Y}=\vec{1}).$$ 
It follows that  $Y\in\vec{W}$ and $\vec{w}(Y)=1$ for all $y
\in \vec{Y}$.

Now let $\nu$ be an assignment to $\vec{X}$ and $\vec{Y}$ such
that $\nu|_{\vec{X}}=\tau$. It cleary suffices to show that $\phi$ is
true under assignment $\tau$.  
Let $\vec{W}^{\prime}= 
\vec{W} - \{ Y\in {\vec{Y}\ |\ \nu(Y)=0}\} $; that is,
$\vec{W}^{\prime}$ contains all the variables $X^{b}$
that are  in $\vec{W}$, and all the variables $Y\in\vec{Y}$
for which $\nu(Y)=1$. 
By AC2(b),  it follows that
$(M,u)\models[\vec{W}^{\prime}\leftarrow\vec{w}]\psi$.
Since $\vec{W}^{\prime}$ contains all the variables $X^{b}$ 
in $\vec{W}$, we have that $(M,u)
\models[\vec{W}^{\prime}\leftarrow\vec{w}]\neg \psi_{1}$.
Thus, we must have that
$(M,u)\models[\vec{W}^{\prime}\leftarrow\vec{w}]
\psi_{3}$. Since
$(M,u)\models[\vec{W}^{\prime}\leftarrow\vec{w}](A=0)$, 
it follows that
$(M,u)\models[\vec{W}^{\prime}\leftarrow\vec{w}]\overline{\varphi}[\vec{X}/\vec{X}^1]$.  

Note that, for $Y \in \vec{Y}$  $\vec{w}(Y) = 1$ iff 
$\nu(Y)=1$; moreover, $\vec{w}(X^1) = 1$ iff $\tau(X) = 1$ iff $\nu(X) = 1$.
Thus, the fact that $(M,u) \models [\vec{W}' \leftarrow
\vec{w}]\psi_3$ implies that $\varphi$ is satisfied by $\nu$, so we
are done.

This completes the proof of the theorem.

\section{Proof of Theorem~\ref{theorem-pi-2}}

Again, as we oberved in the main part of the paper, membership is
straightforward, so we focus on hardness.
We describe a reduction from the language 
$\Pi_{2}(\mbox{SAT})$ to $\Lminbin$, which suffices to prove the result.
The argument is similar in spirit to that for
Theorem~\ref{theorem-lsinglebin-complete}.

Given a CQBF $\forall \vec{Y} \exists \vec{X} \varphi$, consider the
tuple $(M,u,\psi,\zug{A_1,A_2},\zug{0,0})$ where $M = (\U,\V,\R)$ is a
binary causal model and 
\begin{itemize}
\item $\U=\{ U\} $;
\item $\V = \vec{X} \cup \{ Y^{0}\ |\ Y\in\vec{Y}\} \cup \{ Y^{1}\
|\ Y\in\vec{Y}\} \cup \{ A_1, A_2, S\}$, where 
$A_1$, $A_2$, and $S$ are fresh variables;
\item the structural equations for $A_1$ and $A_2$ are $A_1 = S$ and
$A_2 = S$, and, 
for all other variables $V \in \V$, the equation is $V=U$;
\item $u=0$;
\item $\psi = \psi_1 \vee (\psi_2\wedge \psi_3) \vee (S=0)$ where
\begin{itemize}
	\item $\psi_1 = \neg\left(\bigwedge_{Y\in\vec{Y}}(Y^{0}\ne
	Y^{1})\right)$; 
	\item $\psi_2 = \neg(\vec{A} = 1 \wedge \vec{X} = 1)$;
	\item $\psi_3 = (A_1 = A_2) \vee \neg \overline{\varphi}[\vec{Y}/\vec{Y}^1]$.
\end{itemize}
\end{itemize}
We prove that $\forall \vec{Y} \exists \vec{X} \varphi = \T$ iff 
$(M,u,\psi,\vec{A},\vec{0})$ is in $\Lminbin$.


First suppose that $\forall \vec{Y} \exists \vec{X} \varphi = \T$.  To
show that $(M,u,\psi,\vec{A},\vec{0})$ is in $\Lminbin$, 
we must prove that AC1 and AC3 hold.

For AC1, since $(M,u) \models \vec{Y}^0 = \vec{Y}^1 = \vec{0}$, we
clearly have  $(M,u) \models \psi_1$, so $(M,u) \models (\vec{A} =
\vec{0}) \wedge \psi$. 

\commentout{
For AC2, let $\vec{Z} = \vec{A}$, so $\vec{W} = \V = \vec{A}$.
Define $\vec{w}'$ so that $\vec{w}(Y) = 0$ for all variables in $Y \in
\vec{Y}^0$ and $\vec{w}'(V) = 1$ for the remaining variables in $\V$.
Let $\vec{x}' = \zug{1,1}$.
For AC2(a), it is easy to see that $(M,u) \models [\vec{A}
\leftarrow \vec{1}, \vec{W} \leftarrow \vec{w}']\neg \psi$:
$(M,u) \models [\vec{A}
\leftarrow \vec{1}, \vec{W} \leftarrow \vec{w}']\neg \psi_1$, since
$(M,u) \models [\vec{A}
\leftarrow \vec{1}, \vec{W} \leftarrow \vec{w}'](Y^0 \ne Y^1)$ for all
$Y \in \vec{Y}$; 
and clearly $(M,u) \models [\vec{A}
\leftarrow \vec{1}, \vec{W} \leftarrow \vec{w}']\neg \psi_2$,
since $(M,u) \models [\vec{A}
\leftarrow \vec{1}, \vec{W} \leftarrow \vec{w}'](\vec{A} = \vec{1} \wedge
\vec{X} = \vec{1})$.
For AC2(b), fix a subset $\vec{W}^\prime$ of $\vec{W}$.
It is easy to see that $(M, u) \models [\vec{A} \leftarrow
\vec{0}, \vec{W}' \leftarrow
\vec{w}^{\prime}](\psi_2\wedge\psi_3)$ (since
$(M, u) \models [\vec{A} \leftarrow
\vec{0}, \vec{W}' \leftarrow
\vec{w}^{\prime}](\vec{A} = \vec{0})$), 
so $(M, u) \models [\vec{A} \leftarrow \vec{0},
\vec{W}' \leftarrow \vec{w}^{\prime}]\psi$.
}

To show that AC3 holds,
we need to show that neither $A_1 = 0$ nor $A_2 = 0$ is
a cause of $\psi$ in $(M, u)$.  We prove that $A_1 = 0$ is not a cause of
$\psi$ in $(M,u)$; the argument for $A_2 = 0$ not being a cause is
identical.  

It suffices to prove that AC2 does not hold.  So suppose by way of
contradiction that $(\vec{W},\vec{w},1)$ is a witness for $A_1$ being a
cause of $\psi$ in $(M,u)$.   Since AC2(a) holds, we must have 
\begin{equation}\label{eq3}
(M,u) \models [A_1 \leftarrow 1, \vec{W} \leftarrow \vec{w}](\neg \psi_1
\wedge (\neg \psi_2 \vee \neg \psi_3) \wedge (S=1)).
\end{equation}
Thus, $S \in \vec{W}$ and $\vec{w}(S) = 1$ (for otherwise
$(M,u) \models [A_1 \leftarrow 1, \vec{W} \leftarrow \vec{w}](S=0)$).
Moreover, since $(M,u) \models [A_1 \leftarrow 1, \vec{W} \leftarrow
\vec{w}]\neg \psi_1$, 
for all $Y \in \vec{Y}$, either $Y^0 \in \vec{W}$ and 
$\vec{w}(Y^0) = 1$ or $Y^1 \in \vec{W}$ and 
$\vec{w}(Y^1) = 1$, and it is not the case that both $Y^0$ and $Y^1$
are in $\vec{W}$ and $\vec{w}(Y^0) = \vec{w}(Y^1)$.

Now consider $A_2$.  There are three possibilities:
\begin{enumerate}
\item[(a)] $A_2 \in \vec{W}$ and $\vec{w}(A_2) = 0$;
\item[(b)] $A_2 \in \vec{W}$ and $\vec{w}(A_2) = 1$;
\item[(c)] $A_2 \notin \vec{W}$.
\end{enumerate}
We show that we get a contradiction in each case.

If (a) holds, note that since $$(M,u) \models [A_1 \leftarrow 1, \vec{W}
\leftarrow \vec{w}](A_2=0),$$ it follows that 
$(M,u) \models [A_1 \leftarrow 1, \vec{W}
\leftarrow \vec{w}]\psi_2$, so by (\ref{eq3}), 
$(M,u) \models [A_1 \leftarrow 1, \vec{W}
\leftarrow \vec{w}]\neg \psi_3$.  Moreover, since 
$(M,u) \models [A_1 \leftarrow 1, \vec{W}
\leftarrow \vec{w}](A_1 \ne A_2) \land \neg \psi_3$, 
it follows that 
$(M,u) \models [A_1 \leftarrow 1, \vec{W}
\leftarrow \vec{w}]\overline{\varphi}[\vec{Y}/\vec{Y}^1]$.

Let  $Z' = \emptyset$ and let $\vec{W}' = \vec{W} -
\{A_2\}$.  
We show that $(M,u) \models [A_1 \leftarrow 0, \vec{W}'
\leftarrow \vec{w}]\neg \psi$, so that AC2(b) does not hold.
First observe that
$(M,u) \models [\vec{W}' \leftarrow \vec{w}](A_1 \ne
A_2)$.  Since $S$ and all the 
variables in $\vec{X}$, $\vec{Y}^0$, and $\vec{Y}^1$ are in both
$\vec{W}'$ and $\vec{W}$,
it follows from (\ref{eq3}) that 
$$(M,u) \models [A_1 \leftarrow 0, \vec{W}' \leftarrow \vec{w}] 
(\neg  \psi_1 \land \overline{\varphi}[\vec{Y}/\vec{Y}^1] \land
(S=1)).$$  
Thus, $(M,u) \models [A_1 \leftarrow 0, \vec{W}' \leftarrow \vec{w}]
\neg \psi$, and AC2(b) does not hold.

If (b) or (c) hold, define an 
assignment $\nu$ to the variables in $\vec{Y}$ by taking 
$\nu(Y) = 1$ if $Y^1 \in \vec{W}$ and 
$\vec{w}(Y^1) = 1$ and $\nu(Y) = 0$ if $Y^0 \in \vec{W}$ and 
$\vec{w}(Y^0)=1$.
(As we observed above, exactly one of these two cases occurs, so $\nu$
is well defined.)
By assumption,   $\forall \vec{Y} \exists \vec{X} \varphi = \T$, so
there exists an assignment $\tau$ to 
$\vec{X}$ that makes  $\varphi$ true if the assignment to $\vec{Y}$ is
determined by $\nu$.

We again show that AC2(b) does not hold.  Let $Z' = \emptyset$ and 
let $\vec{W}^{\prime} = \vec{W} - \{ X: \tau(X) = 0 \}$.
Since $S \in 
\vec{W}'$ and $\vec{w}(S)=1$, it is easy to see that in both case (b) and (c), 
we have $(M,u) \models [A_1 \leftarrow 0, \vec{W}' \leftarrow
\vec{w}](A_1 \ne A_2)$.  (In case (b), this is becaause
$\vec{w}(A_2) = 1$; in case (c), this is because we have the equation
$A_2 = S$ and $\vec{w}(S) = 1$).  
The definition of $\vec{W}^{\prime}$ ensures that 
$$(M, u) \models [A_1 \leftarrow 0, \vec{W}^\prime \leftarrow
\vec{w}]\overline{\varphi}[\vec{Y}/\vec{Y}^1],$$
so that
$(M, u) \models [A_1 \leftarrow 0,\vec{W}^\prime \leftarrow
\vec{w}]\neg{\psi_3}$, and 
hence also
$(M,u) \models [A_1 \leftarrow 0,\vec{W}^\prime \leftarrow
\vec{w}]\neg{\psi}$, again showing that AC2(b) does not hold.
We conclude that AC3 holds for $\vec{A}$.

Now suppose that $(M,u,\psi,\vec{A},\vec{0})$ is in $\Lminbin$. We must
show that  
$\forall \vec{Y} \exists \vec{X} \varphi(\vec{X},\vec{Y}) = \T $.

Let $\nu$ be some assignment to $\vec{Y}$.
Let $\vec{W} = \{ S \} \cup \vec{X} \cup \vec{Y}^0 \cup \vec{Y}^1$ and define
$\vec{w}$ as follows:
\begin{itemize}
\item $\vec{w}(S) = 1$;
\item $\vec{w}(X) = 1$ for all $X \in \vec{X}$;
\item $\vec{w}(Y^{\nu(y)}) = 1$ and $\vec{w}(Y^{1 - \nu(y)}) = 0$ for
all $Y \in \vec{Y}$.
\end{itemize}

Since AC3 holds, $A_1 \leftarrow 0$ cannot be a cause of $\psi$ in
$(M,u)$ with witness $(\vec{W},\vec{w},1)$.  It is straightforward to
check that $(M, u) \models [A_1 \leftarrow 1,
\vec{W} \leftarrow \vec{w}]\neg\psi$, using the fact that
$\vec{w}(S) = 1$.
Hence, AC2(a) holds for $A_1 \leftarrow 0$.  AC3 holds trivially, and we have
already observed that A1 holds.  Thus, AC2(b) cannot hold for $A_2
\leftarrow 0$, that is, there exist 
$\vec{W}' \subseteq \vec{W}$ and $\vec{Z}' \subseteq \{A_2\}$ such that
$(M, u) \models [A_1 \leftarrow 0,
\vec{W}^{\prime} \leftarrow \vec{w}, \vec{Z}^\prime
\leftarrow \vec{z}^*]\neg{\psi}$. It follows that
\begin{itemize}
\item $S \in \vec{W}^{\prime}$ and $\vec{w}(S) = 1$; and
\item either $Y^0 \in \vec{W}'$ and 
$\vec{w}(Y^0) = 1$ or $Y^1 \in \vec{W}'$ and 
$\vec{w}(Y^1) = 1$, and it is not the case that both $Y^0$ and $Y^1$
are in $\vec{W}'$ and $\vec{w}(Y^0) = \vec{w}(Y^1)$.
\end{itemize}

Since $(M, u) \models [A_1 \leftarrow 0,
\vec{W}^{\prime} \leftarrow \vec{w}, \vec{Z}^\prime
\leftarrow \vec{z}^*]\psi_2$, it must be the case that 
$(M, u) \models [A_1 \leftarrow 0,
\vec{W}^{\prime} \leftarrow \vec{w}, \vec{Z}^\prime
\leftarrow \vec{z}^*]\neg \psi_3$.  This, in turn, implies that 
$(M, u) \models [A_1 \leftarrow 0,
\vec{W}^{\prime} \leftarrow \vec{w}, \vec{Z}^\prime
\leftarrow \vec{z}^*]\overline{\phi}[\vec{Y}/\vec{Y}^1]$.
Now 
$(M, u) \models [\vec{A}' \leftarrow \vec{a}', \vec{W}' \leftarrow
\vec{w}, \vec{Z}' \leftarrow \vec{z}^*]\psi_2$. Thus,
we must have 
$(M, u) \models [\vec{A}' \leftarrow \vec{a}', \vec{W}^{\prime}
\leftarrow \vec{w}, \vec{Z}' \leftarrow \vec{z}^*]\neg{\psi_3}$; thus,
$(M, u) \models [\vec{A}' \leftarrow \vec{a}', \vec{W}^{\prime} \leftarrow \vec{w}, \vec{Z}' \leftarrow \vec{z}^*]\overline{\varphi}[\vec{Y}/\vec{Y}^1]$.
Now define $\tau(X) = 1$ iff
$X \in \vec{W}^{\prime}$. It is immediate that $\tau$ satisfies
$\varphi$ if the values of $Y$ are assigned according to $\nu$.
It follows that 
$\forall \vec{Y} \exists \vec{X} \varphi(\vec{X},\vec{Y}) = \T $, as desired.

}


\paragraph{Acknowledgements:} 
Joseph Halpern was supported in part by NSF grants 
IIS-0911036 and  CCF-1214844, AFOSR grant FA9550-08-1-0438, ARO grant W911NF-14-1-0017, and
by the DoD 
Multidisciplinary University Research Initiative (MURI) program
administered by AFOSR under grant FA9550-12-1-0040.

\bibliographystyle{aaai}

\fullv{
\bibliography{joe,z,fv}
}
\shortv{

}

\end{document}